%% file: main.tex
\documentclass[final,12pt]{colt2025} %

\usepackage[utf8]{inputenc}
\usepackage[T1]{fontenc}
\usepackage{graphicx}
\usepackage{color}
\usepackage{enumerate}
\usepackage{verbatim}
\usepackage{thmtools,thm-restate}
\usepackage{wrapfig}
\usepackage{xspace}
\usepackage{amsfonts}
\usepackage{bbm}
\usepackage{algorithm, algpseudocode, caption, subcaption}
\usepackage{afterpage}

\definecolor{darkgreen}{rgb}{0,0.5,0}
\usepackage{hyperref}
\hypersetup{
    unicode=false,          %
    colorlinks=true,        %
    linkcolor=blue,          %
    citecolor=darkgreen,        %
    filecolor=magenta,      %
    urlcolor=cyan           %
}

\usepackage[disableredefinitions,roman]{complexity}
\usepackage[framemethod=tikz]{mdframed}
\global\mdfdefinestyle{myframe}{leftmargin=.75in,rightmargin=.75in,linecolor=black,linewidth=1.5pt,innertopmargin=10pt,innerbottommargin=10pt} 
\usepackage{framed}
\usepackage{nicefrac}

\date{}

\newtheorem{claim}[theorem]{Claim}

\usepackage[capitalize, nameinlink]{cleveref}
\crefname{theorem}{Theorem}{Theorems}
\Crefname{lemma}{Lemma}{Lemmas}
\Crefname{alg}{Algorithm}{Algorithms}
\Crefname{claim}{Claim}{Claims}
\Crefname{subclaim}{Claim}{Claims}
\Crefname{infclaim}{Claim}{Claims}
\Crefname{observation}{Observation}{Observations}
\Crefname{invariant}{Invariant}{Invariants}
\Crefname{algorithm}{Algorithm}{Algorithms}

\newcommand{\tdmin}{\operatorname{TS}_{\min}}

\DeclareMathOperator*{\argmin}{arg\,min}

\allowdisplaybreaks[1]

\newcommand{\mcC}{\mathcal{C}}
\newcommand{\mcF}{\mathcal{F}}
\newcommand{\mcX}{\mathcal{X}}

\newcommand{\mcA}{\mathcal{A}}
\newcommand{\mcB}{\mathcal{B}}
\newcommand{\AND}{\operatorname{AND}}

\title[Lower Bounds for Greedy Teaching Set Constructions]{Lower Bounds for Greedy Teaching Set Constructions}
\usepackage{times}

\coltauthor{%
 \Name{Spencer Compton} \Email{comptons@stanford.edu}\\
 \addr Stanford University
 \AND
 \Name{Chirag Pabbaraju} \Email{cpabbara@stanford.edu}\\
 \addr Stanford University%
 \AND
 \Name{Nikita Zhivotovskiy} \Email{zhivotovskiy@berkeley.edu}\\
 \addr UC Berkeley%
}
\graphicspath{{img/}}
\begin{document}

\maketitle

\begin{abstract}
A fundamental open problem in learning theory is to characterize the best-case teaching dimension $\tdmin$ of a concept class $\mcC$ with finite VC dimension $d$. Resolving this problem will, in particular, settle the conjectured upper bound on Recursive Teaching Dimension posed by \citeauthor*{simon2015open} (COLT 2015). Prior work used a natural greedy algorithm to construct teaching sets recursively, thereby proving upper bounds on $\tdmin$, with the best known bound being $O(d^2)$ (\citeauthor*{Hu2017}, COLT 2017). In each iteration, this greedy algorithm chooses to add to the teaching set the $k$ labeled points that restrict the concept class the most. In this work, we prove lower bounds on the performance of this greedy approach for small $k$. Specifically, we show that for $k = 1$, the algorithm does not improve upon the halving-based bound of $O(\log(|\mcC|))$. Furthermore, for $k = 2$, we complement the upper bound of $O\left(\log(\log(|\mcC|))\right)$ from \citeauthor*{Moran2015} (FOCS 2015) with a matching lower bound. Most consequentially, our lower bound extends up to $k \le \lceil c d \rceil$ for small constant $c>0$: suggesting that studying higher-order interactions may be necessary to resolve the conjecture that $\tdmin = O(d)$.
\end{abstract}

\section{Introduction}
One of the well-known open problems in learning theory is to characterize the size of the smallest teaching set, called the \emph{best-case teaching dimension} and denoted by $\tdmin$, for a concept class of finite VC dimension. More specifically, given a concept class $\mcC$ of VC dimension $d$ defined on a finite domain $\mathcal{X}$, we ask whether there exists a concept $c \in \mcC$ whose teaching set (i.e., a set of {domain elements} on which $c$ differs from every other concept in $\mcC$) has size $O(d)$. Addressing this question would immediately solve the COLT Open Problem posed by \cite{simon2015open}, where the conjecture is that the \emph{Recursive Teaching Dimension} is $O(d)$. We refer to \citep{doliwa2014recursive} for exact definitions and a detailed account of the topic.

The early works on determining the teaching dimension trace back to \cite{cover1965geometrical}, who showed that $\tdmin = O(d)$ for the concept class induced by half-spaces. In fact, Cover used the geometric structure of half-spaces to bound the \emph{average-case teaching dimension} by $O(d)$, which implies the desired $\tdmin = O(d)$. A generalization of this approach was studied in \citep{doliwa2014recursive}, where the authors analyzed \emph{shortest-path closed} concept classes and showed that their average teaching set has size $O(d)$, implying the same bound on $\tdmin$. However, there is a known limitation to using average teaching set size to bound $\tdmin$, since the average-case teaching dimension need not be bounded by the VC dimension \citep{kushilevitz1996witness, doliwa2014recursive}. A related line of work also leverages additional structure on the concept class, such as intersection-closed classes \citep{kuhlmann1999teaching, doliwa2014recursive}, which likewise guarantee $\tdmin = O(d)$. However, these structural assumptions are class-specific and do not provide a general solution to bounding $\tdmin$ for arbitrary VC classes.

A line of work exploiting \emph{no properties} of the concept class beyond its finite VC dimension was initiated by \cite{kuhlmann1999teaching}, who proved that $\tdmin = 1$ for classes with $d = 1$. For general $d$, there is a simple halving-based bound of $\tdmin = O\bigl(\log(|\mcC|)\bigr)$, which is sensitive to the size of $\mcC$. The first result to improve on this basic result was given by \cite*{Moran2015}, who showed that $\tdmin = O\bigl(d 2^d \log(\log(|\mcC|))\bigr)$. In fact, the dependence on $|\mcC|$ is unnecessary and the bound of \cite{Moran2015} has been subsequently improved to $O(d2^d)$ by \cite*{Chen2016}, and later to the current state-of-the-art bound $O(d^2)$ by \cite*{Hu2017}.

Remarkably, the proofs of all the above-mentioned general bounds on $\tdmin$ rely on the same greedy strategy, formally given as \cref{algo:greedy}. The idea is as follows: pick an integer $k$ (the “greediness” parameter), and at each step find a subset $T^\star \subset \mathcal{X}$ of size $k$ together with a binary pattern $b^\star \in \{0,1\}^k$ such that the number of concepts in $\mcC$ agreeing with $b^\star$ on $T^\star$ is minimized (but still nonempty). Then add $T^\star$ to the current teaching set, and restrict $\mcC$ to those concepts matching \(b^\star\) on \(T^\star\). {This restriction is denoted by $\mcC|_{T^\star, b^\star}$.} Repeat this procedure until the remaining concept class contains exactly one concept, which is then characterized by the constructed teaching set.

As a proof of concept, the result of \cite{kuhlmann1999teaching} for $d = 1$ follows from applying~\cref{algo:greedy} with $k = 1$. The definition of the VC dimension in his proof ensures that the algorithm finishes constructing the teaching set right after the first iteration, thereby establishing $ \tdmin = 1$ when $d = 1$. Similarly, the analysis in \citep{Moran2015} applies~\cref{algo:greedy} with $k = 2$, 
the bound of \cite{Chen2016}\footnote{Using a fixed $k$ will incur a slightly worse guarantee of $O(d^2 2^d)$, compared to their $O(d 2^d)$ guarantee with exponentially decreasing $k$.} with $k = 2^d (d - 1) + 1$, 
while the bound of \cite{Hu2017} can be modified slightly so that it corresponds to the application of the algorithm with $k = O(d)$. 

The main aim of this work is to study the limitations of \cref{algo:greedy} as a general method for constructing teaching sets for different values of $k$. Prior to this work, even for $k = 1$, the size of the constructed teaching sets remained unclear for general concept classes. We show that, in the worst case, \cref{algo:greedy} with $k = 1$ cannot achieve any better dependence on the size of the teaching set than the halving-based bound $O(\log(|\mcC|))$. Since our focus is on lower bounds, we establish the tightness of some previous results, including the somewhat exotic $O(\log(\log(|\mcC|)))$ dependence on the teaching set size in \citep{Moran2015} when the algorithm is run with $k = 2$.  We summarize our findings in the following table:

\begin{table}[h]
\renewcommand*{\arraystretch}{1.3}
\centering
\setlength{\tabcolsep}{8pt} %
\small
\begin{tabular}{|c|p{5cm}|p{5cm}|}
\hline
& \textbf{Upper Bound} & \textbf{Lower Bound} \\
\hline
$k=1$ & 
$O\bigl(\log |\mcC|\bigr)$; folklore & 
$\Omega\bigl(\log |\mcC|\bigr)$; \cref{theorem:rectangle-lower-bound} \\
\hline
$2 \leq k \leq \lceil cd \rceil$ for some $c > 0$ & 
$O_{d,k}\bigl(\log(\log |\mcC|)\bigr)$; \citep{Moran2015}\footnotemark & 
$\Omega_{d,k}\bigl(\log(\log |\mcC|)\bigr)$; \cref{theorem:general-k-lower-bound} \\
\hline
$k = \lceil c' d \rceil$ for a larger $c' > 0$ & 
$O\bigl(d^2\bigr)$; \citep{Hu2017} & 
$\Omega\bigl(d\bigr)$; trivial \\
\hline
\end{tabular}
\vspace{1ex}
\noindent $O_{d,k}(\cdot)$ denotes ignoring multiplicative factors in $d,k$. Prior to our work, only the trivial lower bound of $\Omega(d)$ (e.g., $\mcC = \{0, 1\}^d$) was known for all settings.
\end{table}

\footnotetext{The upper bound in \citep{Moran2015} is shown for $k = 2$, but their proof implies the same bound for all $k > 2$.}

{The main consequence of our result is to refute the plausible-seeming agenda of resolving the $\tdmin = O(d)$ conjecture by more sharply analyzing the natural greedy~\cref{algo:greedy} for smaller $k=o(d)$; we show that ~\cref{algo:greedy} may fail to construct small teaching sets even when $k = \lceil c\,d \rceil$ for a sufficiently small absolute constant $c > 0$.}
This suggests that a better study of higher-order interactions, or some approach that exploits the overall structure of the concept class, might be necessary. 
Note how this does not imply the greedy approach is suboptimal for large $k$; by definition, the greedy algorithm optimally finishes in one round when $k=\tdmin$. 
Our construction reveals an unexpected phase transition: if indeed $\tdmin = O(d)$ holds, then by selecting \emph{a sufficiently large} constant $c' > 0$, \cref{algo:greedy} with $k = \lceil c'\,d \rceil$ should be capable of constructing the desired teaching set in a single iteration, as dictated by its definition. However, our findings show that for $k = \lceil c\,d \rceil$ with \emph{a small} constant $c$, the algorithm fails to achieve the desired bound.

The remainder of the paper is organized as follows. In \cref{sec:lowerkone} we analyze the basic case of \( k = 1 \), and in~\cref{sec:largerk} we consider other values of \( k \). %
Both concept classes have small $\tdmin$: they are barriers for the greedy algorithm with small $k$, not for the general $\tdmin = O(d)$ conjecture (see \cref{app:tdmin}). To keep the paper concise, it is assumed that the reader is familiar with standard definitions and results, such as VC dimension; further details can be found in standard textbooks.

\input{rectangles}

\input{large-k}

\acks
This work was supported by the National Defense Science \& Engineering Graduate
(NDSEG) Fellowship Program, Tselil Schramm’s NSF CAREER Grant no. 2143246, Gregory
Valiant’s and Moses Charikar's Simons Foundation Investigator Awards, and NSF award AF-2341890. Nikita Zhivotovskiy would like to thank Benny Sudakov and Istv\'{a}n Tomon for a number of insightful discussions regarding upper bounds for teaching sets.

\bibliography{ref}

\clearpage
\input{supplement}

\end{document}

%% file: rectangles.tex
\section{Lower bound for $k=1$}
\label{sec:lowerkone}
Our first main result is a lower bound on the size of the teaching set returned by \cref{algo:greedy} and the procedure \textsc{Greedy} for $k = 1$. The geometric construction employed in this proof serves both as a foundation and as an illustration for our analysis when $k \ge 2$, which is presented in \cref{sec:largerk}. In fact, the proof of~\cref{theorem:rectangle-lower-bound} is driven by the illustration in \Cref{fig:rectangles-lb}. Once the construction and the procedure of \cref{algo:greedy} for $k = 1$ are understood, the remainder of the section is devoted to formalizing the intuitively clear argument.

Before we present our statement and construction, we recall that each rectangle classifies points in its interior and along its border as $1$, and points in its exterior as $0$.
\begin{algorithm}[t]
\SetAlgoLined
\caption{Greedy algorithm for constructing teaching sets}\label{algo:greedy}
\KwIn{Concept class $\mcC$, greediness parameter $k$}
\KwOut{Teaching set $S$ for some concept $c \in \mcC$}
\SetKwFunction{Greedy}{Greedy}
\SetKwProg{Proc}{Procedure}{:}{}
\Proc{\Greedy{$\mcC, k$}}{
    $S \gets \emptyset$\; \par
    \While{$|\mcC| > 1$}{\label{line:greedy-loop}
        $T^\star, b^\star \gets \displaystyle \argmin_{\substack{T \subseteq \mcX,\, 1 \le |T| \le k \\ b \in \{0,1\}^{|T|} \\ \left|\mcC|_{T,b}\right|>0}} \bigl|\mcC|_{T, b}\bigr|$\; \label{line:greedy-choice}
        \textcolor{blue}{\Comment{Greedily compute smallest restriction\footnotemark}}\; \par
        $S \gets S \cup T^\star$\; \label{line:append-to-S}
        \textcolor{blue}{\Comment{Add $T^\star$ to teaching set}}\; \par
        $\mcC \gets \mcC|_{T^\star, b^\star}$\; \label{line:restrict-C}
        \textcolor{blue}{\Comment{Restrict $\mcC$ to teaching set constructed so far}}\; \par
    }
    \Return{$S, c$}\; {(where $\mcC=\{c\}$)}\; \par
}
\footnotetext{Here, we may assume that ties are broken in favor of a $T$ that has smaller size.}
\end{algorithm}
\footnotetext{Here, we may assume that ties are broken in favor of a $T$ that has smaller size.}

\begin{theorem}[Rectangles Lower Bound for $k=1$]
    \label{theorem:rectangle-lower-bound}
    There exists a family $\{\mcF_N\}$ of concept classes (here $N = 1, 2, \ldots $) such that
    \begin{enumerate}
        \item $\mcF_N$ consists of indicators of axis-aligned rectangles in $\mathbb{R}^2$ (i.e., VC dimension at most 4), %
        \item $\mcF_N$ has size $2^{\Theta(N)}$ and is defined on a domain $\mcX \subseteq \mathbb{R}^2$ of size $2^{\Theta(N)}$,
        \item \textsc{Greedy}$(\mcF_N, 1)$ returns a teaching set of size at least $N = \Omega(\log(|\mcF_N|))$.
    \end{enumerate}
\end{theorem}

The remainder of the section is devoted to the proof of \cref{theorem:rectangle-lower-bound}.

\begin{figure}[t]
    \centering
    \includegraphics[scale=0.4]{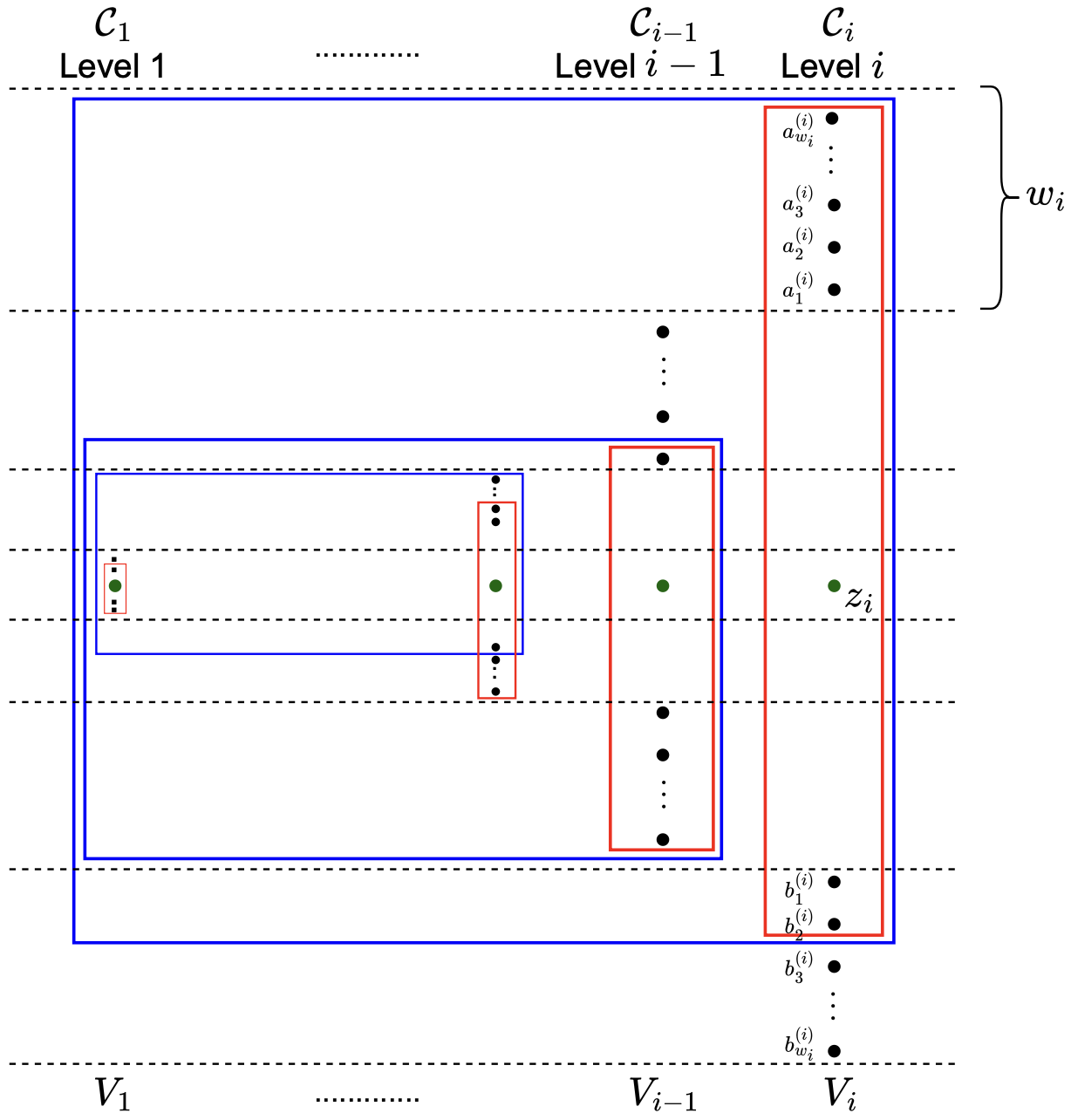}
    \caption{The arrangement of the point sets $V_1 \cup V_2 \cup \dots$ in the two-dimensional plane. The black horizontal dashed lines delineate the ``vertical ranges'' of these sets. The red rectangles denote concepts in $\mcC_i^{(\text{up, 1})}$ and $\mcC_i^{(\text{down, 1})}$, whereas the blue rectangles denote concepts in $\mcC_i^{(\text{up, 2})}$ and $\mcC_i^{(\text{down, 2})}$. For example, observe how the red rectangle in Level $i$, which belongs to $\mcC_i^{(\text{up, 1})}$, is enlarged to include all the points in $V_1 \cup \dots \cup V_{i-1}$, and yields the corresponding blue rectangle from $\mcC_i^{(\text{up, 2})}$.}
    \label{fig:rectangles-lb}
\end{figure}
\paragraph{Construction of the class.} We first describe the construction of the concept class $\mcF_N$ for every $N \ge 1$. %
\textit{Domain.} Consider a collection of sets of points $V_1 \cup \dots \cup V_N$. Each $V_i$ consists of a center point $z_i$, $w_i \triangleq 2^{10i}$ points extending vertically up, and $w_i$ points extending vertically down; thus, $|V_i|=2w_i+1$. Each set $V_i$ will be placed horizontally next to $V_{i+1}$ (see \Cref{fig:rectangles-lb}). The domain $\mcX \subseteq \mathbb{R}^2$ will precisely be $\mcX = V_1 \cup V_2 \cup \dots \cup V_N$, so that 
\begin{align*}
    &|\mcX|=\sum_{i=1}^N|V_i|=\sum_{i=1}^N(2^{10i+1}+1) \le \sum_{i=1}^N2^{10i+2} \le 2^{10N+3}, \\
    \text{and also }&|\mcX| = \sum_{i=1}^N(2^{10i+1}+1) \ge 2^{10N+1}.
\end{align*}
Thus, $|\mcX|=2^{\Theta(N)}$ as required. 

\textit{Concept class.} We now describe the concept class $\mcF_N$. We construct $\mcF_N$ as the union of $\mcC_1,\mcC_2,\dots,\mcC_N$, where we think of $\mcC_i$ to be defined at the ``$i^\text{th}$ level''.  It is helpful to refer to \Cref{fig:rectangles-lb}. To define $\mcC_i$, denote the $w_i$ points vertically above the center $z_i$ as $a^{(i)}_1,\dots,a^{(i)}_{w_i}$, and those vertically below $z_i$ as $b^{(i)}_1,\dots,b^{(i)}_{w_i}$. Let $\mcC_i^{(\text{up, 1})}$ consist of all rectangles that precisely contain $a^{(i)}_1,\dots,a^{(i)}_{w_i}$, the center $z_i$, and then additionally a (potentially empty) prefix of $b^{(i)}_1,\dots,b^{(i)}_{w_i}$. Similarly, let $\mcC_i^{(\text{down, 1})}$ consist of all rectangles that precisely contain $b^{(i)}_1,\dots,b^{(i)}_{w_i}$, the center $z_i$, and then additionally a (potentially empty) prefix of $a^{(i)}_1,\dots,a^{(i)}_{w_i}$. We have that $|\mcC_i^{(\text{up, 1})}|=|\mcC_i^{(\text{down, 1})}|=w_i+1$, and both contain a common rectangle that contains all the points in $V_i$.

Consider enlarging every rectangle in $\mcC_i^{(\text{up, 1})}$ to additionally  include all the points in $V_1 \cup V_2 \cup \dots \cup V_{i-1}$. These enlarged rectangles form $\mcC_i^{(\text{up, 2})}$. Similarly, enlarge every rectangle in $\mcC_i^{(\text{down, 1})}$ to include all the points in $V_1 \cup  V_2 \cup \dots \cup V_{i-1}$. These enlarged rectangles form $\mcC_i^{(\text{down, 2})}$. Again, we have that $|\mcC_i^{(\text{up, 2})}|=|\mcC_i^{(\text{down, 2})}|=w_i+1$, and both contain a common rectangle that contains all the points in $V_1 \cup \dots \cup V_i$.

The concept class $\mcC_i$ is then simply $\mcC_i^{(\text{up, 1})} \cup \mcC_i^{(\text{up, 2})} \cup \mcC_i^{(\text{down, 1})} \cup \mcC_i^{(\text{down, 2})}$. Subtracting out the common concepts, we have that $|\mcC_i|=4(w_i+1)-2=4w_i+2$. This gives us that
\begin{align*}
    &|\mcF_N| = \sum_{i=1}^N|\mcC_i| = \sum_{i=1}^N (4w_i+2) = \sum_{i=1}^N (2^{10i+2}+2) \le \sum_{i=1}^N 2^{10i+3} \le 2^{10N+4}, \\
    \text{and also }&|\mcF_N| = \sum_{i=1}^N (2^{10i+2}+2) \ge 2^{10N+2}.
\end{align*}
Thus, $|\mcF_N|=2^{\Theta(N)}$ also as required. Note also that every rectangle in $\mcC_i$ contains the center $z_i$, and either all the points in $V_1 \cup V_2 \cup \dots \cup V_{i-1}$ or none of them. Additionally, any point in $V_1 \cup V_2 \cup \dots \cup V_{i-1}$ is contained in exactly half the rectangles in $\mcC_i$, and any point in $V_i$ that is not the center is contained in at least $2w_i+2$ rectangles in $\mcC_i$.

The following property of the construction will also be useful ahead. It says that the number of concepts in the $i^\text{th}$ level is much larger than the \textit{total} number of concepts in all lower levels.

\begin{claim}[A level dominates all lower levels]
    \label{claim:level-domination}
    For any $i \in \{2,\dots,N\}$,
    \begin{align*}
        \sum_{j=1}^{i-1}|\mcC_j| <  \frac{1}{2}|\mcC_i|.
    \end{align*}
\end{claim}
\begin{proof}
    We have that
    \begin{align*}
        & \sum_{j=1}^{i-1}|\mcC_{j}| = \sum_{j=1}^{i-1}(4w_j+2) = \sum_{j=1}^{i-1}(2^{10j+2}+2) \\
        &\le 2^{10(i-1)+4} 
        < \frac{1}{2} (4 \cdot 2^{10i}+2) = \frac{1}{2} (4 w_i+2) = \frac{1}{2}|\mcC_i|.
    \end{align*}
\end{proof}

We can show a lower bound for the teaching set constructed by the greedy algorithm with greediness parameter $k=1$, when it is instantiated for our constructed concept class $\mcF_N=\bigcup_{i=1}^N\mcC_i$.

\begin{lemma}
    The teaching set returned by $\textsc{Greedy}(\mcF_N, 1)$ has size at least $N$.
\end{lemma}
\begin{proof}
    For $i=0,1,2,\dots,N-1$, we claim that at the beginning of the $i^\text{th}$ iteration of the while loop of \Cref{algo:greedy} (where $i=0$ refers to the first iteration), 
    \begin{align}
        \label{eqn:inductive-hypothesis}
        \mcC=\bigcup_{j=1}^{N-i}\mcC_i \quad \text{and } S=\{z_{N-i+1},z_{N-i+2},\dots,z_{N}\}.
    \end{align}
    We argue this inductively. When $i=0$, we are just entering the while loop for the very first time, and so $\mcC=\mcF_N=\bigcup_{j=1}^{N}\mcC_i$, and also $S=\emptyset$. Now, suppose that the claim holds for some $i \ge 0$: we will show that it continues to hold for $i+1$.
    In particular, we will argue that in the $i^\text{th}$ iteration of the while loop, $T^\star$ is chosen to be $\{z_{N-i}\}$ and $b^\star$ to be $0$ in \Cref{line:greedy-choice}. This will prove the claim, since (i) $T^\star$ gets appended to $S$ in \Cref{line:append-to-S}, (ii) all the concepts in $\mcC_{N-i}$ get removed from $\mcC$ upon restricting to $T^\star, b^\star$ in \Cref{line:restrict-C}, since every concept in $\mcC_{N-i}$ labels $z_{N-i}$ (which is the center of $V_{N-i}$) as $1$, and (iii) no concepts in $\mcC_1,\dots,\mcC_{N-i-1}$ are removed, since all such concepts label $\{z_{N-i}\}$ as $0$.

    For any $x$, let $\mcC(x,0)$ and $\mcC(x,1)$ denote the concepts in $\mcC$ that label $x$ as 0 and 1 respectively; here, $\mcC=\bigcup_{j=1}^{N-i}\mcC_i$ is the effective concept class at the beginning of the $i^\text{th}$ iteration of the while loop. Observe that for any $x \notin V_1 \cup \dots \cup V_{N-i}$, $\mcC(x,1)=\emptyset$ (such an $x$ is strictly to the right of the remaining rectangles in $\mcC$).  Thus, $T=\{x\}, b=1$ is not under consideration for the greedy choice. Furthermore, observe also that $\mcC(x,0)=\mcC$. Since $\mcC$ has at least two concepts (by virtue of entering the loop), these two concepts must then differ on some $y \in V_1 \cup \dots \cup V_{N-i}$, which means that the $\argmin$ can also not be attained at $T=x,b=0$ (since, e.g., restricting to $T=\{y\}, b=1$ reduces the size of the class by at least 1).

    We can thus restrict our attention to $x \in V_1 \cup \dots \cup V_{N-i}$. Note that $|\mcC(z_{N-i}, 0)| > 0$ (the rectangles to the left of $z_{N-i}$ do not contain it). Note also that $|\mcC(z_{N-i}, 0)| < |\mcC(z_{N-i}, 1)|$ because all the concepts in $\mcC_{N-i}$ label $z_{N-i}$ as 1. Even though all the remaining concepts label $z_{N-i}$ as 0, we still know from \Cref{claim:level-domination} that these are strictly less than half the number of concepts in $\mcC_{N-i}$.

    Hence, our task is to show that for any $x \neq z_{N-i}$, $|\mcC(z_{N-i},0)| < |\mcC(x, 0)|$ and $|\mcC(z_{N-i},0)| < |\mcC(x, 1)|$. If this is the case, then $T=\{z_{N-i}\}, b=0$ %
    will indeed realize the greedy choice in \Cref{line:greedy-choice}. Note that this is equivalent to showing that $|\mcC(z_{N-i},1)| > |\mcC(x, 1)|$ and $|\mcC(z_{N-i},1)| > |\mcC(x, 0)|$.
    
    \medskip\noindent\textbf{Case 1:} $x \in V_1 \cup V_2 \cup \dots \cup V_{N-i-1}$, that is $x$ is in a lower level than $z_{N-i}$. \\
    \noindent Recall that exactly half of the concepts in $\mcC_{N-i}$ label such an $x$ as 1. Then, even if all the remaining concepts were to label $x$ as 1, we still have that
    \begin{align*}
        |\mcC(x,1)| &\le \frac{1}{2}|\mcC_{N-i}| + \sum_{j=1}^{N-i-1}|\mcC_{j}| < \frac{1}{2}|\mcC_{N-i}| + \frac{1}{2}|\mcC_{N-i}| = |\mcC_{N-i}| = |\mcC(z_{N-i}, 1)|,
    \end{align*}
    where in the second inequality, we used \Cref{claim:level-domination}, and in the last equality, we used that all concepts in $\mcC_{N-i}$ label $z_{N-i}$ as 1. By exactly the same calculation, we also get that $|\mcC(x,0)|<|\mcC(z_{n-i}, 1)|$.

    \medskip\noindent\textbf{Case 2:} $x \in V_{N-i}$, that is, $x$ is in the same level as $z_{N-i}$ but is not the center. \\
    We immediately have that $|\mcC(z_{N-i},1)| > |\mcC(x, 1)|$, because only the concepts in $\mcC_{N-i}$ label these points as 1, and the concepts in $\mcC_{N-i}$ that label $x$ as 1 are a strict subset of the concepts that label $z_{N-i}$ as 1 (which is all of them). Furthermore, it is also the case that $|\mcC(z_{N-i},1)| > |\mcC(x, 0)|$. To see this, observe how $x$ is 0 in strictly less than half of $C_{N-1}$, so that
    \begin{align*}
        |\mcC(x,0)| &< \frac{1}{2}|\mcC_{N-i}| + \sum_{j=1}^{N-i-1}|\mcC_{j}| < \frac{1}{2}|\mcC_{N-i}| + \frac{1}{2}|\mcC_{N-i}| = |\mcC_{N-i}| = |\mcC(z_{N-i}, 1)|. %
    \end{align*}

    This completes our inductive proof for \eqref{eqn:inductive-hypothesis}. In particular, for $i=N-1$, we have shown that $S=\{z_2,z_3,\dots,z_{N}\}$, and $\mcC = \mcC_1$. Again, $\mcC_1$ has at least two concepts that differ on some point in $V_1$, and so, the $(N-1)^{\text{th}}$ iteration of the while loop will find such a point in $V_1$ to add to $S$. Thus, the final teaching set that is returned has size at least $N$.
\end{proof}

We now reflect on the structural properties of the construction that the proof effectively relied on. The number of concepts in level $i$ heavily dominates (\Cref{claim:level-domination}) the total number of concepts in lower levels. Every concept in level $i$ labels the center $z_i$ as 1, which biases it. Since the concepts in lower levels form such a minority, they don't sway the bias of the center by much. On the other hand, for any point in a lower level, the concepts in level $i$ maintain exactly a 50-50 balance. The rest of the concepts might sway this bias by a bit, but again, these concepts are too few when compared to the concepts in level $i$, so such points stay nearly unbiased.

%% file: large-k.tex
\section{Lower bound for $k\ge 2$}
\label{sec:largerk}
   We will now show that for any $k \ge 2$, there is a concept class with VC dimension $d \le 4k+1$ for which the $\log(\log(|\mcC|))$ dependence is unavoidable. This dependence on the concept class size appeared in \citep{Moran2015}, where the upper bound 
\[
\tdmin = O\Bigl(d\,2^d\log\bigl(\log(|\mcC|)\bigr)\Bigr)
\]
was proven for $k = 2$. In particular, our next result shows that there exists a family of concept classes $\{\mcF_N\}$, each having VC dimension at most $9$, such that \textsc{Greedy}$(\mcC, 2)$ outputs a teaching set of size $\Omega\bigl(\log\bigl(\log(|\mcF_N|)\bigr)\bigr)$.
        
    Our concept class will take inspiration from our $k=1$ construction, although we require a more complicated construction for this setting.

    \begin{figure}[t]
    \centering
        \includegraphics[scale=0.85]{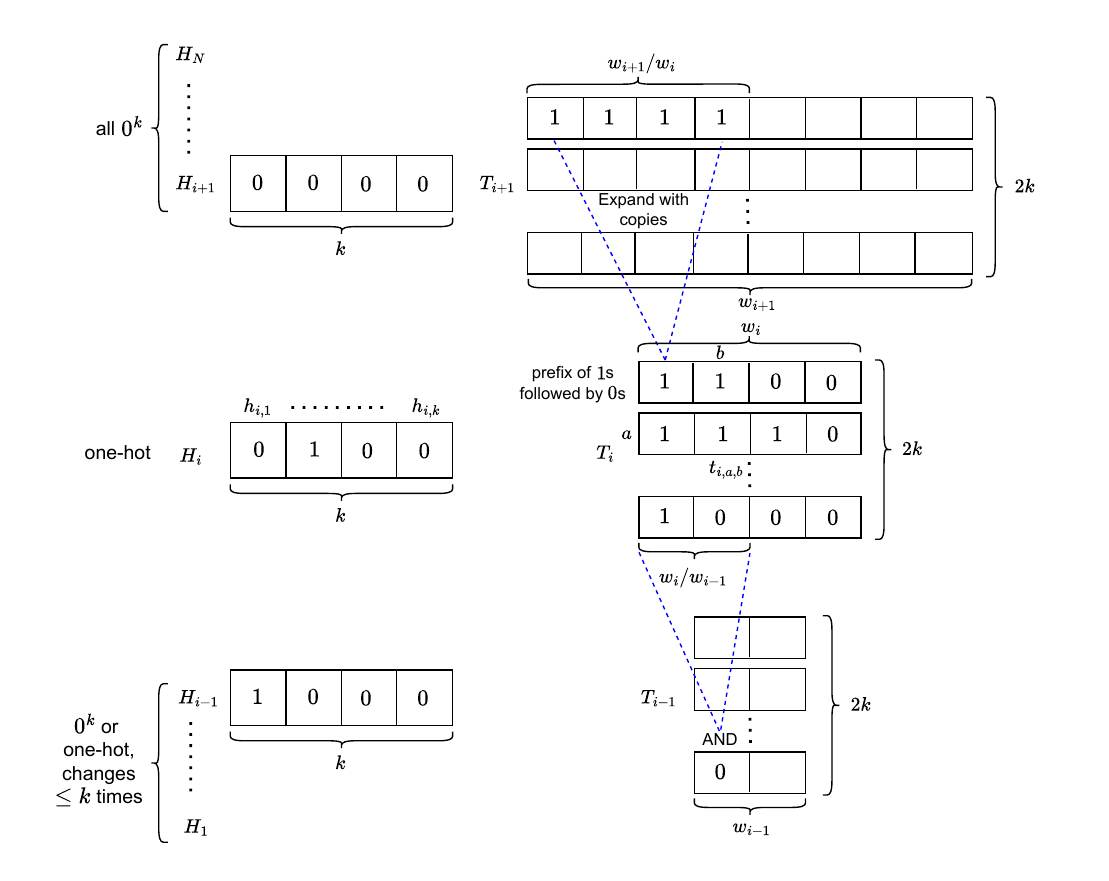}
        \caption{Example of a concept in $\mcC_i$ for \cref{theorem:general-k-lower-bound}. Given the prefixes for $T_i$, the values are deterministically expanded/contracted for other $T_j$. At most $k$ values of $j \in \{1,\dots,i-1\}$ may have $H_j$ be different from $H_{j+1}$; in this concept a change occurred from $H_{i-1}$ to $H_{i}$.}
        \label{fig:general-lb}
\end{figure}

    \begin{theorem}[Lower Bound for $k \ge 2$]
    \label{theorem:general-k-lower-bound}
    For every positive integer $k \ge 2$, there exists a family $\{\mcF_N\}$ of concept classes (here $N = 1, 2, \ldots$) such that
    \begin{enumerate}
        \item $\mcF_N$ has VC dimension at most $4k+1$, 
        \item $\mcF_N$ has size at most {$2^{4k\log(8k) \cdot 2^{2N}}$} and is defined on a domain $\mcX$ of size at most {$6k \cdot 2^{\log(8k) \cdot 2^{2N}}$},
        \item \textsc{Greedy}$(\mcF_N, k)$ returns a teaching set of size at least {$kN = \Omega\left( \log(\log(|\mcF_N|))\right)$}.
    \end{enumerate}
\end{theorem}

    \paragraph{Construction of the class.} We begin by designing the concept class $\mcF_N$ for every $N \ge 1$.
    
    \textit{Domain. } Our domain $\mcX$ {is an abstract finite set that} consists of the union of $N$ sets $V_1 \cup  \dots \cup V_N$. Each $V_i$ is the union of two sets of domain points $H_i$ and $T_i$. The set $H_i$ has $k$ \textit{head points} $h_{i,1},\dots,h_{i,k}$. The set $T_i$ has $2k$ rows $T_{i,1},\dots,T_{i,2k}$ of $w_i$ \textit{tail points} each, where  $t_{i,a,b}$ denotes the tail point in the $a^\text{th}$ row and $b^\text{th}$ column of $T_i$ (see~\cref{fig:general-lb}).
    We set $w_i \triangleq 2^{\log(8k) \cdot 2^{2i}}$. Note that $w_{i+1} = w_i^4$; in particular, each $w_i$ divides $w_{i+1}$, a property we will utilize. {Roughly, the head points $H_i$ will play a similar role to the center point $z_i$ in \cref{theorem:rectangle-lower-bound}, and the tail points $T_i$ will play a similar role to $V_i \setminus z_i$ in \cref{theorem:rectangle-lower-bound}.}
    
    \textit{Concept class. } The concept class $\mcF_N$ will be a union of $\mcC_1,\dots,\mcC_N$. Let us focus on the concepts in $\mcC_i$. We will specify these as a product of a concept class $\mcA_i$ on the combined set of head points $H_1,\dots,H_N$, and concept class $\mcB_i$ on the combined set of tail points $T_1,\dots,T_N$. That is, $\mcC_i \triangleq \mcA_i \otimes \mcB_i$, {yielding a concept for every pair of concepts in $\mcA_i$ and $\mcB_i$.}

    The concept class $\mcA_i$ consists of all concepts $c_h$, where: (i) $c_h$ labels every point in $H_{i+1},\dots,H_{N}$ as 0, (ii) $c_h$ labels $H_i$ as a one-hot vector (that is, $c_h(h_{i,j})=1$ for exactly one $j \in [k]$), (iii) for any $j < i$, $c_h$ labels $H_j$ either as the zero vector $0^k$, or as a one-hot vector, and (iv) there are at most $k$ ``changes'' in the labeling of $c_h$ on $H_1,\dots,H_{i-1}$; concretely, there are at most $k$ indices in $\{1,2,\dots,i-1\}$ where the labeling of $c_h$ on $H_j$ (as a $k$-dimensional vector) differs from the labeling on $H_{j+1}$. {This is illustrated in the left half of \cref{fig:general-lb}. Property (iv)} is enforced to control the VC dimension of $\mcA_i$ to be $O(k)$. %

    We now describe the concept class $\mcB_i$.  {We refer to the right half of \cref{fig:general-lb} for better understanding}. A concept $c_t$ in $\mcB_i$ labels each row of $T_i$ with a (possibly empty) prefix of 1s, followed by a (possibly empty) suffix of 0s---in total, there are $(w_i+1)^{2k}$ possible choices of these prefix sizes over the $2k$ rows of $T_i$. We will have a concept $c_t$ in $\mcB_i$ for each such choice; moreover, the labels of $c_t$ on $T_i$ will determine its labels on $T_1,\dots,T_{i-1}$ as well as $T_{i+1},\dots,T_N$. Thus, the total size of $\mcB_i$ will be $(w_i+1)^{2k}$. To describe the labels that a concept $c$ realizes on these other tail sets, suppose that it labels the $2k$ rows of $T_i$ as $(c_t(t_{i,1,1}),\dots,c_t(t_{i,1,w_i})),\ldots,(c_t(t_{i,2k,1}),\dots,c_t(t_{i,2k,w_i}))$. 
    
    For $j>i$, starting with $j=i+1$, the labeling of $c_t$ on $T_{j,a}$ will be such that each $c_t(t_{j,a,b})$ is one of $\frac{w_j}{w_{j-1}}$ copies of the label that $c_t$ assigns to a corresponding point in $T_{j-1,a}$. More concretely, for $j > i$,
    \begin{align*}
        c_t(t_{j, a, b}) = c_t\left(t_{j-1, a, \left\lceil \frac{b}{w_j/w_{j-1}} \right\rceil}\right).
    \end{align*}
    Similarly, for $j < i$, starting with $j=i-1$, each label that $c_t$ assigns to a point in $T_{j,a}$ will be the logic $\AND$ of the labels it assigns to a batch of $\frac{w_{j+1}}{w_j}$ points in $T_{j+1,a}$. Concretely, for $j < i$,
    \begin{align*}
        c_t(t_{j, a, b})  = \AND\left(\left\{c_t\left(t_{j+1, a, \frac{w_{j+1}}{w_j}(b-1) +1}\right),\ldots,c_t\left(t_{j+1, a, \frac{w_{j+1}}{w_j}(b-1) +\frac{w_{j+1}}{w_j}}\right)\right\}\right).
    \end{align*}
    This construction is better understood visually via \cref{fig:general-lb}. We can see how a label of $c_t$ in a particular row in $T_i$ is copied over (``expanded'') to $w_{i+1}/w_i$ many slots in the corresponding row in $T_{i+1}$. Similarly, we can also see how the label of $c_t$ in a row in $T_{i-1}$ corresponds to the $\AND$ (``contraction'') of $w_i/w_{i-1}$ labels in the corresponding row in $T_i$. We note how these expansion/contraction operations maintain the property that every row of tail points is labeled with a prefix of 1s followed by a suffix of 0s.

    We recall that $\mcC_i = \mcA_i \otimes \mcB_i$, and $\mcF_N = \bigcup_{i=1}^N\mcC_i$. We note also that the concept classes $\mcC_i$ are disjoint, as is witnessed by how for $i<j$, any concept in $\mcC_i$ labels $H_j$ as $0^k$, while every concept in $\mcC_j$ labels $H_j$ as some one-hot vector.

    \paragraph{Size of teaching set. } We now analyze the size of the teaching set that \textsc{Greedy}$(\mcF_N, k)$ returns.

    \textit{Intuition. } Our strategy will be similar to the proof of \Cref{theorem:rectangle-lower-bound}. We will aim to maintain that at the beginning of the $i^\text{th}$ iteration of the while loop in \textsc{Greedy}$(\mcF_N, k)$, the remaining concepts are exactly $\mcC_1 \cup \dots \cup \mcC_{N-i}$. In particular, we will inductively prove that at iteration $i$, the algorithm picks $T^\star$ to be the $k$ head points in $H_{N-i}$, and sets $b^\star=0^k$. This removes all concepts in $\mcC_{N-i}$, yet none of the concepts in $\mcC_1, \dots, \mcC_{N-i-1}$. Proceeding thus for $N$ iterations would then force the returned teaching set to be of size at least $kN$ as desired. Our main intuition is as follows: for any choice of restriction other than $H_{N-i}, 0^k$ in \Cref{line:greedy-choice}, there are at least $k$ rows of $T_{N-i}$ that are completely unrestricted, and hence these contribute at least $(w_{N-i}+1)^k$ concepts consistent with the restriction. However, our choice of $w_{N-i}$ ensures that $(w_{N-i}+1)^k$ is strictly greater than the number of concepts in $\mcC_1 \cup \dots \cup \mcC_{N-i-1}$; but these are precisely the concepts that remain if our restriction forces all of $H_i$ to $0$. We now make this formal.

    For $i=0,1,2,\dots,N-1$, we claim that at the beginning of the $i^\text{th}$ iteration of the while loop in \Cref{line:greedy-loop} (where $i=0$ refers to the first iteration), 
    \begin{align}
        \label{eqn:inductive-hypothesis-general}
        \mcC=\bigcup_{j=1}^{N-i}\mcC_i \quad \text{and } S=H_{N-i+1} \cup H_{N-i+2} \cup \dots \cup H_{N}.
    \end{align}
    When $i=0$, we are just entering the while loop for the very first time, and so $\mcC=\mcF_N=\bigcup_{j=1}^{N}\mcC_i$, and also $S=\emptyset$. Now, suppose that the claim holds for some $i \ge 0$: we will show that it continues to hold for $i+1$.
    In particular, we will argue that in the $i^\text{th}$ iteration of the while loop, $T^\star$ is chosen to be $H_{N-i}$ and $b^\star$ to be $0^k$ in \Cref{line:greedy-choice}. This will prove the claim, since (i) $T^\star$ gets appended to $S$ in \Cref{line:append-to-S}, (ii) all the concepts in $\mcC_{N-i}$ get removed from $\mcC$ upon restricting to $T^\star, b^\star$ in \Cref{line:restrict-C}, since every concept in $\mcC_{N-i}$ labels the head points $H_{N-i}$ as a one-hot vector, and (iii) no concepts in $\mcC_1,\dots,\mcC_{N-i-1}$ are removed, since all these concepts label $H_{N-i}$ as $0^k$.

    So, let $T \subseteq \mcX, 1 \le |T| \le k$ and $b \in \{0,1\}^{|T|}$ be any candidate choice for the $\argmin$ in \Cref{line:greedy-choice}. Note that this also requires that at least one concept in $\mcC_1 \cup \dots \cup \mcC_{N-i}$ labels $T$ as $b$. Let us decompose $T$ into head and tail points as follows:
    \begin{align}
        \label{eqn:decomposition-head-tail}
        T = \{\underbrace{x_1,\dots,x_{n_h}}_{\text{head points}}, \underbrace{y_1,\dots,y_{n_t}}_{\text{tail points}}\},
    \end{align}
    where $0 \le n_h,n_t \le |T|$ and $n_h + n_t = |T| \le k$. Similarly, let $b_h \in \{0,1\}^{n_h}$, $b_t \in \{0,1\}^{n_t}$ be the labeling in $b$ for the head and tail points respectively. Then, it must be the case that there is some $j \in \{1,2,\dots,N-i\}$, some concept $c_h \in \mcA_j$, and some concept $c_t \in \mcB_j$, such that $c_h$ labels $\{x_1,\dots,x_{n_h}\}$ as $b_h$ and $c_t$ labels $\{y_1,\dots,y_{n_t}\}$ as $b_t$.

    \begin{claim}
        \label{claim:tail-set}
        Consider any arbitrary labeling $b_t$ of $\{y_1,\dots,y_{n_t}\}$ (where $n_t \le k$) that is consistent with some $c_t \in \mcB_1 \cup \dots \cup \mcB_{N-i}$. Then, there are at least $(w_{N-i}+1)^k$ different concepts in $\mcB_{N-i}$ that are consistent with this labeling.
    \end{claim}
    \begin{proof}
        Each $y_i$ belongs to some set $T_j$ of tail points---let the row in $T_j$ that contains $y_i$ be denoted as $r_i$. Then consider the rows $r_1,r_2,\dots,r_{n_t}$. These are at most $n_t \le k$ distinct rows. In particular, there are at least $k$ rows from $\{1,2,\dots,2k\}$ that do not feature in $r_1,r_2,\dots,r_{n_t}$. 

        Now, observe how $\mcB_{N-i}$ actually has additional structure: the labels on different rows of $T_1,\dots,T_N$ do not interact with each other. Namely, if we denote by $\mcB_{N-i, r}$ the restrictions of the concepts in $\mcB_{N-i}$ to row $r$ in $T_1,\dots,T_N$, then $\mcB_{N-i} = \mcB_{N-i, 1} \otimes \mcB_{N-i, 2} \otimes \dots \otimes \mcB_{N-i, 2k}$. The labeling $b_t$ on $\{y_1,\dots,y_{n_t}\}$ possibly pins down $n_t \le k$ sets in this product, in the worst case, to a single concept (there will always be at least one concept, since, e.g., for all the $y_i$s that are in some common row $r_i$, there is at least one concept in $\mcB_{N-i, r_i}$ that is consistent with the labels on these $y_i$s, because these labels must have---due to realizability---the necessary prefix structure enforced on concepts in $\bigcup_{j=1}^{N-i}\mcB_j$); but even so, there are at least $k$ sets that are untouched. Each of these sets has size $w_{N-i}+1$ (we choose a prefix in the corresponding row in $T_{N-i}$, and expand/contract upward/downward). Thus, the total cardinality of the product (and hence the number of concepts in $\mcB_{N-i}$), even after restricting to $b_t$ on $\{y_1,\dots,y_{n_t}\}$, is at least $(w_{N-i}+1)^k$ as claimed.
    \end{proof}

    \begin{claim}
        \label{claim:head-set}
        Consider any arbitrary labeling $b_h$ of $\{x_1,\dots,x_{n_h}\}$ (where $n_h \le k$, and excluding the case where $\{x_1,\dots,x_{n_h}\}=H_{N-i}$ and $b_h=0^k$) that is consistent with some $c_h \in \mcA_1 \cup \dots \cup \mcA_{N-i}$. Then, there is at least one concept in $\mcA_{N-i}$ that is consistent with this labeling.
    \end{claim}
    \begin{proof}
         Let $S=\{x_1,\dots,x_{n_h}\}$; note that any $x_i \in S$ that belongs to $H_{N-i+1} \cup \dots H_N$ must be labeled as 0 in $b_h$. This is because $S$ is realizable by $\mcA_1 \cup \dots \cup \mcA_{N-i}$, and every concept in this union labels all head points above $H_{N-i}$ as 0. So, from the point of view of proving the existence of a concept in $\mcA_{N-i}$ consistent with $b_h$ of $S$, we can assume without loss of generality that $S \subseteq H_{1} \cup \dots H_{N-i}$. We will then construct a labeling on $H_1,\dots,H_N$ consistent with $b_h$ on $S$, such that this labeling corresponds to a concept in $\mcA_{N-i}$. The labeling on each of $H_{N-i+1},\dots,H_N$ is simply the zero vector. We specify the labeling on $H_{N-i},\dots,H_1$ according to cases:

        \medskip\paragraph{Case 1: $S$ consists of $k$ points from some $H_j$, for $1 \le j \le N-i$. } Note then that $b_h$ can either be the zero vector, or a one-hot vector---this is because $b_h$ is a labeling of $S$ that is realizable by $\mcA_1 \cup \dots \cup \mcA_{N-i}$, and every concept in this union labels a head set either as the zero vector, or a one-hot vector. If $b_h$ is a one-hot vector, then simply assign each of $H_{N-i},\dots,H_1$ to this one-hot vector. Otherwise, if $b_h$ is the zero vector, then by the assumption in the claim, it must be the case that $j < N-i$. We then label $H_{N-i}$ by some arbitrary one-hot vector, and each of $H_{N-i-1},\dots,H_1$ as the zero vector. Either way, we ensure that the assignment to $H_{N-i},\dots,H_1$ incurs at most 1 change, and hence the overall labeling to $H_1,\dots,H_N$ is a valid concept in $\mcA_{N-i}$.
        
        \medskip\paragraph{Case 2: $S$ has less than $k$ points from each of $H_1,\dots,H_{N-i}$. } Consider 
        $$
        I = \{j \le N-i: \text{$S$ contains at least one point of $H_j$}\}.$$ 
        Note that $|I| \le n_h \le k$. Furthermore, for any such $j \in I$, at most one element in $S$ that is in $H_j$ can be labeled $1$ in $b_h$. For every $j \in I$, if among the elements in $S$ that are in $H_j$, there is an $x$ that is labeled as 1 in $b_h$, we assign $H_j$ the one-hot vector that labels $x$ as 1. Otherwise, we assign $H_j$ to be the one-hot vector where the 1 is at an arbitrary head point in $H_j$ not in $S$ (such a point exists because $S$ has strictly less than $k$ points from any head). Now, consider the indices in $\{N-i,N-i-1,\ldots,1\}$ that remain to be assigned a label (these are precisely the indices not in $I$). The indices in $I$ induce a partition of $\{N-i,N-i-1,\ldots,1\}$ into at most $k+1$ groups. Namely, if $I = \{j_1,\dots,j_{|I|}\}$ in decreasing order, these groups are $\{N-i,\dots, j_1\}$, $\{j_1-1,\dots,j_2\}, \dots, \{j_{|I|}-1,\dots,1\}$. For any every $j_l \in I$, consider the partition ending in $j_l$: we label the head set at every index in this partition identically as the one-hot vector we assigned to $H_{j_l}$. Finally, we label every head set in the last partition $\{j_{|I|}-1,\dots,1\}$ identically also to the one-hot vector assigned to $H_{j_{|I|}}$. We can verify that this leads to assigning a one-hot vector to each of $H_{N-i},\dots,H_{1}$, in a way that incurs at most $k$ changes. Thus, our overall labeling to $H_1,\dots,H_N$ corresponds to a valid concept in $\mcA_{N-i}$.
    \end{proof}

    Now, consider any scenario other than when $T=H_{N-i}, b=0^k$. For the decomposition of $T$ as in \eqref{eqn:decomposition-head-tail}, we already argued that there must exist some $j \in \{1,2,\dots,N-i\}$, $c_h \in \mcA_j$ and $c_t \in \mcB_j$ such that $c_h$ labels $\{x_1,\dots,x_{n_h}\}$ as $b_h$ and $c_t$ labels $\{y_1,\dots,y_{n_t}\}$ as $b_t$. Then, $\{y_1,\dots,y_{n_t}\}$ together with the labeling $b_t$ satisfy the condition of \Cref{claim:tail-set}. Thus, there are at least $(w_{N-i}+1)^k$ different concepts in $\mcB_{N-i}$ that are consistent with this labeling. Similarly, $\{x_1,\dots,x_{n_h}\}$ together with the labeling $b_h$ satisfy the condition of \Cref{claim:head-set}. Thus, there is at least one concept in $\mcA_{N-i}$ that is consistent with this labeling. Since $\mcC_{N-i} = \mcA_{N-i} \otimes \mcB_{N-i}$, we can conclude that there are at least $(w_{N-i}+1)^k$ concepts in $\mcC_{N-i}$ consistent with the labeling $b$ on $T$.

    We will now argue that the choice of $T=H_{N-i}, b=0^k$ retains strictly less than $(w_{N-i}+1)^k$ concepts from $\mcC=\cup_{j=1}^{N-i}\mcC_i$. In particular, recall that every concept in $\mcC_{N-i}$ labels $H_{N-i}$ as a one-hot vector. Thus, all these concepts are removed in \Cref{line:restrict-C}. The number of remaining concepts can then be at most $\sum_{j=1}^{N-i-1}|\mcC_j|$. As with the rectangles construction in the section above, it is also the case here that the number of concepts in $\mcC_i$ dominates the total number of concepts in $\mcC_1,\dots,\mcC_{i-1}$. In \cref{sec:domination-appendix}, we show an even stronger domination result implying $\sum_{j=1}^{N-i-1}|\mcC_j| < (w_{N-i}+1)^k$:

    \begin{claim}[$\mcC_{i}$ dominates $\mcC_1,\dots,\mcC_{i-1}$]
    \label{claim:domination-general-case}
        For any $i \in \{1,\dots,N\}$,
        \begin{equation*}
            \sum_{j=1}^{i} |\mcC_{j}| \le w_{i}^{4k}. %
        \end{equation*}
        In particular, for $i \in \{2,\dots,N\}$, $\sum_{j=1}^{i-1} |\mcC_{j}| \le w_{i-1}^{4k} < \left(w_{i} + 1\right)^k < \left(w_{i} + 1\right)^{2k} = |\mcB_i| < |\mcC_i|$.
    \end{claim}
    \begin{proof}
        Since $\mcC_j = \mcA_j \otimes \mcB_j$, we have that $|\mcC_j|=|\mcA_j||\mcB_j|$. Recall that $|\mcB_j|=(w_j+1)^{2k}$. To bound $|\mcA_j|$, we recall the conditions (i),(ii), (iii) and (iv) from above that a concept in $\mcA_j$ must satisfy. In particular, a concept can make labeling changes in at most $k$ locations in $\{1,2,\dots,j-1\}$---there are at most $(j-1)^k$ possible choices for these locations. Each choice of change locations partitions $H_1,\dots,H_j$ into at most $k+1$ buckets, where the labeling on each bucket should be the same. Furthermore, we may label each bucket with one of $k+1$ choices, corresponding to $0^k$ and one of $k$ one-hot vectors (with the exception of the last bucket that includes $H_j$, which may only be assigned a one-hot vector). In total, we have that
        \begin{align*}
            |\mcC_j| = |\mcA_j||\mcB_j| = (w_j+1)^{2k}|\mcA_j| \le  (w_j+1)^{2k}(j-1)^k(k+1)^{k+1} %
            \le (4jkw_j)^{2k},
        \end{align*}
        so that for any $i \in \{1,\dots,N\}$,
        \begin{align*}
            \sum_{j=1}^{i} |\mcC_{j}| \le (4ki)^{2k}\sum_{j=1}^{i}w_j^{2k} \le (8kiw_{i})^{2k}.
        \end{align*}
        In the last inequality, we used that $w_{j+1} \ge 2w_j$. Finally, using that for any $j \ge 1$, $8kj \le 2^{\log(8k)\cdot2^{2j}}=w_j$, we get that
        \begin{align}
            \label{eqn:step:pow-bound}
            \sum_{j=1}^{i} |\mcC_{j}| \le w_{i}^{4k}.%
        \end{align}
        In particular, for $i \in \{2,\dots,N\}$,
        \begin{align*}
            \sum_{j=1}^{i-1} |\mcC_{j}| \le w_{i-1}^{4k} = \left(w_{i-1}^{4}\right)^k = w_i^k < (w_i + 1)^k < |\mcC_i|. 
        \end{align*}
    \hfill$\blacksquare$

We remark that our choice of $w_i = 2^{\log(8k) \cdot 2^{2i}}$ was made so as to satisfy $(8k(i-1)w_{i-1})^{2k} \le w_i^k$ in the calculation above.
    \end{proof}
    
    With \cref{claim:domination-general-case}, this completes the inductive proof of \eqref{eqn:inductive-hypothesis-general}. In particular, for $i=N-1$, we have shown that $S=H_2 \cup H_3 \cup \dots H_N$, and $\mcC = \mcC_1$. In the $(N-1)^{\text{th}}$ iteration, the algorithm will choose some $k$ points from $\mcX \setminus S$ to add to the teaching set $S$ (repeating a point that is already in $S$ is strictly suboptimal, as is choosing less than $k$ points). Thus, the final teaching set that is returned has size at least $kN$.

    \paragraph{VC dimension of the concept class. } We will show that the VC dimension of $\mcF_N$ is at most $4k+1$. Roughly, this follows from how any shattered set may not contain too many head points, or we could choose a labeling that forces more than $k$ changes, nor may the set contain too many tail points, or we could choose an impossible labeling due to the prefix structure. 

    \begin{lemma}
    The VC dimension of $\mcF_N$ is at most $4k+1$
\end{lemma}
\begin{proof}
We separately bound the number of head points and tail points in any shattered set:

    \begin{claim}[Few Head Points]
        \label{claim:shatter-head}
        Any set shattered by $\mcF_N$ may contain at most $2k+1$ head points.
    \end{claim}
    \begin{proof}
        If the shattered set contains at least two points from a single head $H_i$, then we know that this set cannot be shattered, as no concept in the class labels both these points simultaneously as $1$. Thus, the set can only contain at most a single point from each head. Let these head points be $\{x_1,\dots,x_m\}$---we will now construct a label pattern on these points that cannot be realized by $\mcF_N$ if $m$ is larger than $2k+1$. Without loss of generality, suppose that the points are sorted in decreasing order of the index of the corresponding head that they appear in (i.e., $x_1 \in H_{i_1},\dots,x_m \in H_{i_m}$, where $i_1 > i_2 > \dots > i_m$). Let us pair up these points into $\lfloor\frac{m}{2} \rfloor$ pairs as $(x_1,x_2), (x_3,x_4),\dots,(x_{2\lfloor\frac{m}{2} \rfloor-1},x_{2\lfloor\frac{m}{2} \rfloor})$. We will determine labels for points in each pair based on the columns that these points lie in (see \cref{fig:general-lb}). If $x_i$ and $x_{i+1}$ are in the same column, then we will ask for $x_i$ to be labeled as 1, and $x_{i+1}$ to be labeled as $0$. Otherwise, $x_i$ and $x_{i+1}$ are in different columns, and we will ask for both of them to be labeled as 1. 
        Notice that the suggested label pattern necessitates a label change at the corresponding heads at each pair; since there are $\lfloor\frac{m}{2} \rfloor$ pairs, and no concept is allowed more than $k$ label changes, $\lfloor\frac{m}{2} \rfloor \le k \implies m \le 2k+1$.
    \end{proof}

    We move on to arguing that no shattered set can contain too many tail points. For this, we will require a structural property about the labels that can be realized at the same row $a \in \{1,2,\dots,2k\}$ for two different tail sets $T_i$ and $T_j$, where $i < j$. For $y \in \{1,2,\dots,w_j\}$, let $f(i, j, y)$ denote the column in $T_{i,a}$ that the point $t_{j,a,y}$ contracts down to; namely $f(i, j, y) \triangleq \left\lceil \frac{y}{w_j/w_i}\right\rceil$.

    \begin{observation}\label{observation:f-and}
        For any concept $c$ in $\mcF_N$, integers $1 \le i < j \le N$, row $1 \le a \le 2k$, and column $1 \le b \le w_i$, it holds that
        \begin{equation}
            \label{eqn:f-and}
            c(t_{i,a,b}) = \AND(\{c(t_{j, a, y}): f(i,j,y)=b\}).
        \end{equation}
    \end{observation}
    \begin{proof}
        First, suppose $c \in \mcC_{1} \cup \dots \cup \mcC_i$. Then, by the way that $c$ is constructed, the label that it assigns to $t_{i,a,b}$ is copied over to all the points $\{t_{j, a, y}: f(i,j,y)=b\}$, and hence \eqref{eqn:f-and} holds. Now, suppose that $c \in \mcC_{j} \cup \dots \cup \mcC_N$. Again, by construction, the labels that $c$ assigns to $\{t_{j, a, y}: f(i,j,y)=b\}$ are contracted all the way down via $\AND$s to $t_{i,a,b}$. Finally, suppose that $c \in \mcC_{l}$ where $l \in \{i+1,\dots,j-1\}$. This case essentially follows by combining the reasoning for the preceding two cases. In more detail, consider the ``contraction path'' of the set $Y=\{t_{j,a,y}:f(i,j,y)=b)\}$ down to the point $t_{i,a,b}$---this path intersects $T_{l,a}$ at a subset of columns $S \subset \{1,2,\dots,w_l\}$, such that every $d \in S$ maps to a distinct batch $E_d$ of $w_j/w_l$ points in $Y$, where these batches are disjoint, and together comprise all of $Y$. Furthermore, the label that $c$ assigns to $t_{l,a,d}$ gets copied out back up to $E_d$. Because of this, $c(t_{l,a,d})$ is indeed the $\AND$ of the labels that $c$ assigns to $E_d$. Moreover, the labels that $c$ assigns to $T_{l,a}$ at the columns in $S$ are also contracted down to $t_{i,a,b}$ via $\AND$s. Together, we get that $c(t_{i,a,b})=\AND(\{c(t_{j,a,y}):f(i,j,y)=b\})$, as desired.
    \end{proof}
    We can now conclude that no set that is shattered by $\mcF_N$ may have more than $2k$ tail points.
    
    \begin{claim}[Few Tail Points]
        \label{claim:shatter-tail}
        Any set shattered by $\mcF_N$ may contain at most $2k$ tail points. 
    \end{claim}
    \begin{proof}
        For the sake of contradiction, if a shattered set contains at least $2k+1$ tail points, then there must be at least two points that correspond to the same row $a$; these points are either within the same $T_i$, or across some $T_i$ and $T_j$. Consider a pair of two such points, $t_{i,a, b_1}$ and $t_{j,a,b_2}$, for $i \le j$. If $i = j$, then without loss of generality, supposing $b_1 < b_2$, it is impossible for $t_{i,a,b_1}$ to be labeled 0 while $t_{j,a,b_2}$ is labeled 1---this is because of the prefix nature of how concepts label rows of tail points. Otherwise, suppose $i < j$. We will use \cref{observation:f-and} to show how it is not possible to realize at least one pattern of labels on $t_{i,a,b_1}$, $t_{j,a,b_2}$

        \medskip\noindent\textbf{Case 1: $f(i,j,b_2) \le b_1$. } 
        We claim that no concept simultaneously labels $t_{j,a,b_2}$ as 0 and $t_{i,a,b_1}$ as 1. To see this, consider a concept $c$ that labels $t_{j,a,b_2}$ as 0, and let $f(i,j,b_2)=b_3 \le b_1$. From \Cref{observation:f-and}, we then know that $c(t_{i,a,b_3})=0$. We then conclude that $c$ cannot label $t_{i,a,b_1}$ as 1, by the prefix nature required of $c$. %

        \medskip\noindent\textbf{Case 2: $f(i,j,b_2) > b_1$. } We claim that no concept simultaneously labels $t_{j,a,b_2}$ as 1 and $t_{i,a,b_1}$ as 0. To see this, consider a concept $c$ that labels $t_{j,a,b_2}$ as 1. Because $f(i,j,b_2) > b_1$, by \Cref{observation:f-and}, we know that $c(t_{i,a,b_1})$ is an $\AND$ of labels that $c$ assigns to a batch of points strictly to the left of $t_{j,a,b_2}$. All these labels must be 1 by the prefix nature of $c$, and hence we conclude $c(t_{i,a,b_1})=1$.
    \end{proof}
    Combining \cref{claim:shatter-head,claim:shatter-tail}, we conclude that any set shattered by $\mcF_N$ must be of size at most $4k+1$, which bounds the VC dimension of $\mcF_N$ by $4k+1$.
    \end{proof}
    
    This proves Part 1 of \Cref{theorem:general-k-lower-bound}.

    \paragraph{Size of the concept class and domain. } Using the domination result of \Cref{claim:domination-general-case}, we get
    \begin{align}
        \label{eqn:general-C-size-bound}
        |\mcF_N| = \sum_{j=1}^N|\mcC_j| \le w_N^{4k} = 2^{4k\log(8k) \cdot 2^{2N}}. %
    \end{align}
    Similarly, the size of the domain $\mcX$ is at most 
    \begin{align*}
        \sum_{i=1}^N (k+2kw_{i}) = k\sum_{i=1}^N (2w_{i}+1) \le 3k\sum_{i=1}^N w_{i} \le 6kw_N = 6k \cdot 2^{\log(8k) \cdot 2^{2N}}.%
    \end{align*}   
    This proves Part 2 of \Cref{theorem:general-k-lower-bound}.

    \paragraph{Concluding $kN = \Omega(\log(\log(|\mcF_N|)))$.} Finally, we may conclude how $|\mcF_N|$ relates to $kN$ as,
    \begin{align*}
        \frac{1}{14} \cdot \log(\log(|\mcF_N|)) &\le \frac{1}{14} \cdot (\log(4k\log(8k)) + 2N) \qquad (\text{using \eqref{eqn:general-C-size-bound}}) \\
        & \le \frac{1}{14} \cdot (\log(4kN \cdot \log(8kN))+2kN) \le \frac{1}{14} \cdot (12kN + 2kN) = kN.
    \end{align*}
    This completes the proof of Part 3 of \Cref{theorem:general-k-lower-bound}, completing its proof.

%% file: supplement.tex
\appendix

\section{Bounding $\tdmin$ for our constructions}\label{app:tdmin}

Our constructions are not counterexamples to the general $\tdmin = O(d)$ conjecture as they have small $\tdmin$. 

\textit{Concept class from \cref{theorem:rectangle-lower-bound} ($k=1$).} The $\tdmin$ for this class is at most 2: only one concept simultaneously labels $z_1$ as $1$, and the point immediately above as $0$.

\textit{Concept class from \cref{theorem:general-k-lower-bound} ($k \ge 2$).} The $\tdmin$ for this class is at most $2k+1$: set the rightmost point in the $k$ rows of $T_N$ as 1, and set one point to $1$ in each of $H_{N-k},\dots,H_{N}$ in a way that forces all the allowed changes and thus determines all other values.